\documentclass[12pt]{l4dc2021}

\usepackage{algorithm}
\usepackage{algorithmic}

\usepackage{xcolor}

\usepackage[utf8]{inputenc} 
\usepackage[T1]{fontenc}    
\usepackage{hyperref}       
\usepackage{url}            
\usepackage{booktabs}       
\usepackage{nicefrac}       
\usepackage{microtype}      

\usepackage{enumerate} 
\usepackage{enumitem} 

\usepackage{multirow}

\newtheorem{thm}{Theorem}[section]

\newtheorem{problem}{Problem}

\usepackage{hyperref,cleveref}

\usepackage{xcolor}

\usepackage{hyperref}

\usepackage{graphicx}
\usepackage{wrapfig}

\usepackage{footmisc}

\title[Transfer Learning without Knowing the Unobserved Context]{Learning without Knowing: Unobserved Context in Continuous Transfer Reinforcement Learning}
\usepackage{times}

\author{%
	\Name{Chenyu Liu}  
	\Email{chenyu.liu018@duke.edu}\\
	\Name{Yan Zhang} \Email{yan.zhang2@duke.edu}\\
	\Name{Yi Shen} \Email{yi.shen478@duke.edu}\\
	\Name{Michael M. Zavlanos} \Email{michael.zavlanos@duke.edu} \\
	\addr Department of Mechanical Engineering and Material Science \\
	Duke University, Durham, NC 27708, USA%
}

\begin{document}

\maketitle

\begin{abstract}%
	In this paper, we consider a transfer Reinforcement Learning (RL) problem in continuous state and action spaces, under unobserved contextual information. For example, the context can represent the mental view of the world that an expert agent has formed through past interactions with this world. We assume that this context is not accessible to a learner agent who can only observe the expert data. Then, our goal is to use the context-aware expert data to learn an optimal context-unaware policy for the learner using only a few new data samples. Such problems are typically solved using imitation learning that assumes that both the expert and learner agents have access to the same information. However, if the learner does not know the expert context, using the expert data alone will result in a biased learner policy and will require many new data samples to improve. 
	To address this challenge, in this paper, we formulate the learning problem as a causal bound-constrained Multi-Armed-Bandit (MAB) problem. The arms of this MAB correspond to a set of basis policy functions that can be initialized in an unsupervised way using the expert data and represent the different expert behaviors affected by the unobserved context. On the other hand, the MAB constraints correspond to causal bounds on the accumulated rewards of these basis policy functions that we also compute from the expert data. The solution to this MAB allows the learner agent to select the best basis policy and improve it online. And the use of causal bounds reduces the  exploration variance and, therefore, improves the learning rate.
	We provide numerical experiments on an autonomous driving example that show that our proposed transfer RL method improves the learner's policy faster compared to existing imitation learning methods and enjoys much lower variance during training.
\end{abstract}


\section{Introduction}\label{sec:intro}

Reinforcement Learning (RL) has been widely used recently to solve stochastic multi-stage decision making problems in high dimensional spaces with unknown models, in areas as diverse as robotics \cite{zhu2017target} and medicine \cite{raghu2017continuous}.
Various methods have been proposed for this purpose that generally require large amounts of data to learn good control policies. Among these methods, Transfer Learning (TL) aims at addressing data efficiency in RL by using data from expert agents to learn optimal policies for learner agents with only a few new data samples \cite{taylor2009transfer,TLsurvey}. Nevertheless, a common assumption in TL methods is that both the expert and learner agents have access to the same information. This is not always the case in practice. For example, human drivers often develop a mental view of the world through past experiences that along with other sensory information, e.g., visual information, guides their actions. This additional experiential information cannot be obtained using, e.g., camera and Lidar sensors found onboard autonomous vehicles. As a result, using expert data without knowing the expert context can generate learned policies that are biased and may require many new data samples to improve.  

In this paper, we propose a new transfer RL method in continuous state and action spaces that can deal with unobservable contextual information. 
Specifically, we model the expert and learner agents using contextual Markov Decision Processes (MDP) \cite{contextual_MDP}, where the state transitions and the rewards are affected by contextual variables. Moreover, we assume that the expert agent can observe this contextual information and can generate data using a context-aware policy. The goal is to use the experience from the expert agent to enable a learner agent, who cannot observe the contextual information, to learn an optimal context-unaware policy with much fewer new data samples. 
In related literature, the contextual variables in contextual MDPs are typically used to model different tasks that the agents can accomplish. In these problems, transfer RL enables the use of expert data from many different source tasks to accelerate learning for a new desired target task, see, e.g., \cite{transfervaluefunction,transfersamples,transferpolicyfunction,barreto2017successor,transfersample2}. 
These methods typically assume that some or all source tasks share the same reward or transition functions with the target task. As a result,  the models learned from the expert's data can be easily transfered to the  learner agent. However, as discussed in \cite{causal_mab}, when contextual information about the world is unobserved, the constructed reward or transition models can be biased, and this can introduce errors in the policies learned by the learner agent.

Related is also recent work on Learning from Demonstrations (LfD); see, e.g., \cite{Gail,multi_modal,infogail,BC,lee2019efficient,ghasemipour2020divergence}. Specifically, \cite{Gail,BC} propose a method to learn a policy maximizing the likelihood of generating the expert's data, while \cite{ghasemipour2020divergence} propose a method to learn a policy that matches the expert-specified state distribution. However, these works assume that the expert and learner agents observe the same type of information. When the expert makes decisions based on additional contextual variables that are unknown to the learner, the imitated policy function can be suboptimal, as we later discuss in Section~\ref{sec:problem}. On the other hand, \cite{multi_modal,infogail,lee2019efficient} propose a method to learn a policy parameterized by latent contextual information included in the expert's dataset. Therefore, given a contextual variable, the learner agent can directly execute this policy. However, if the context can not be observed, these methods cannot be used to learn a context-unaware policy. LfD with contextual information  is also considered in \cite{de2019causal,etesami2020causal,zhang2020causal}. Specifically, \cite{de2019causal,etesami2020causal} assume that the expert's observations are accessible to the learner so that the environment model can be identified. On the other hand, \cite{zhang2020causal} provide conditions on the underlying causal graphs that guarantee the existence of a learner policy matching the expert's performance. However, such a policy does not necessarily exist for general TL problems where unobservable contextual information may affect the expert agent's decisions. 
Along with LfD, related is also literature on one/few-shot reinforcement learning \cite{duan2017one,Yu2019NIPS,Ghasemipour2019NIPS}, where the objective is to retrain an  trained policy for a new unseen scenario using only one or a few new expert demonstrations for this new scenario. However, in the TL problem considered here, no expert demonstrations are available for the unseen testing cases in which the contextual information is hidden and sampling of the contextual variable cannot be controlled.

To the best of our knowledge, most closely related to the problem studied in this paper are the works in \cite{causal_mab,transferRL}. Specifically, \cite{causal_mab} propose a TL method for Multi-Armed-Bandit (MAB) problems with unobservable contextual information using the causal bounds. Then, \cite{transferRL} extend this framework to general RL problems by computing the causal bounds on the value functions. However, the approach in \cite{transferRL} can only handle discrete state and action spaces and does not scale to high dimensional RL problems.
In this paper, we propose a novel TL framework under unobservable contextual information which can be used to solve continuous RL problems. Specifically, we first learn a set of basis policy functions that represent different expert behaviors affected by the unobserved context, using the unsupervised clustering algorithms in \cite{multi_modal,infogail,lee2019efficient}.
Although the return of each data sample is included in the expert dataset, as discussed in \cite{causal_mab}, it is impossible to identify the expected returns of the basis policy functions when the context is unobservable. Therefore, the basis policy functions cannot be transferred directly to the learner agent. Instead, we use them as arms in a Multi-Armed Bandit (MAB) problem that the learner agent can solve to select the best basis policy function and improve it online. Specifically, we propose an Upper Confidence Bound (UCB) algorithm to solve this MAB that also utilizes causal bounds on the expected returns of the basis policy functions computed using the expert data. These causal bounds allow to reduce the exploration variance of the UCB algorithm and, therefore, accelerate its learning rate, as shown also in  \cite{causal_mab}. We provide numerical experiments on an autonomous driving example that show that our proposed TL framework requires much fewer new data samples than existing TL methods that do not consider the presence of unobserved contextual information and learn a single policy function using the expert dataset.

The rest of this paper is organized as follows. In Section \ref{sec:problem}, we define the proposed TL problem in continuous state and action spaces with unobservable contextual information. In Section \ref{sec:methods}, we formulate the MAB problem for the learner agent and develop the proposed causal bound-constrained UCB algorithm to solve it. In Section \ref{sec:experiment}, we provide numerical experiments on an autonomous driving example that illustrate the effectiveness of our method. Finally, in Section \ref{sec:conclusion}, we conclude our work. 

\section{Preliminaries and Problem Definition}\label{sec:problem}
In this section, we define the contextual Markov Decision Processes (MDP) used to model the RL agents, and formulate the TL problem that uses experience data from a context-aware expert agent to learn an optimal policy for a context-unaware learner agent. 
\subsection{Contextual MDPs}
Consider a contextual MDP defined by the tuple $\left(s_t, a_t, f^U(s_{t+1}|s_t, a_t), r^U(s_t, a_t), \rho_{0}\right)$ as in \cite{contextual_MDP}, where $s_t \in \mathcal{S}$ represents the state of the agent at time $t$, and $a_t \in \mathcal{A}$ denotes the action of the agent at time $t$. Unlike~\cite{transferRL}, in this paper, we assume that the state and action spaces $\mathcal{S}$ and $\mathcal{A}$ can be either discrete or continuous. The transition function $ f^U(s_{t+1}|s_t, a_t)$ represents the probability of transitioning to state $s_{t+1}$ after applying action $a_t$ at state $s_t$, the reward function $r^U(s_t, a_t)$ denotes the reward received by the agent after applying action $a_t$ at state $s_t$, and $\rho_{0}: \mathcal{S} \rightarrow \mathbb{R}$ represents the distribution of the initial state $s_{0}$. In a contextual MDP, the transition and reward functions are parameterized by a context variable $U \in \mathcal{U}$, where  $\mathcal{U}$ is a discrete or continuous space. The random variable $U$ is sampled from a stationary distribution $P(U)$ at the beginning of each episode. The transition and reward functions are parameterized by contextual variables in many practical control and learning problems. 
For example, in human-in-the-loop control \cite{zhou2020human}, the context $U$ can represent hidden human preference and affect the reward function. In what follows, we define the context-aware policy function of the expert by $\pi_e(s_t,a_t,U) :\mathcal{S} \times \mathcal{A} \times \mathcal{U} \rightarrow [0,1]$, that represents the probability of selecting action $a_t$ at state $s_t$ given the contextual variable $U$.
We also assume that the expert implements a context-aware policy $\pi_e^\ast$ that approximately solves the problem $\max_{\pi_d} \; \mathbb{E}_{\rho_{0}} \big[\sum_{t=0}^{T} \gamma^{t}r^U(s_{t}, a_{t}) | \pi_e(s_t, a_t, U) \big] \text{ for all } U \in \mathcal{U}$,
where $\gamma \leq 1$ is the discount factor. Note that the expert policy $\pi_e^\ast$ does not need to solve the RL problem exactly. This way we can model, e.g., human experts that occasionally make mistakes. Finally, we assume that the learner agent cannot observe the contextual information $U$. Therefore, the learner implements a context-unaware policy function $\pi(s_t, a_t) :  \mathcal{S} \times \mathcal{A} \rightarrow [0,1]$ that denotes the probability of selecting action $a_t$ only based on state $s_t$. The goal of the learner agent is to find an optimal policy $\pi^\ast(s_t, a_t)$ which solves the problem
\vspace{-2mm}
\begin{align} \label{eqn:LearnerProblem}
	\max_{\pi} \; \mathbb{E}_{\rho_{0}, P(U)} \big[ \sum_{t=0}^{T} \gamma^{t} r^U(s_{t}, a_{t})  |\pi(s_t, a_t)  \big].
\end{align}

\begin{figure}
	\centering
	\subfigure[\footnotesize Expert]{\label{subfig:demostrator}\includegraphics[scale=0.35]{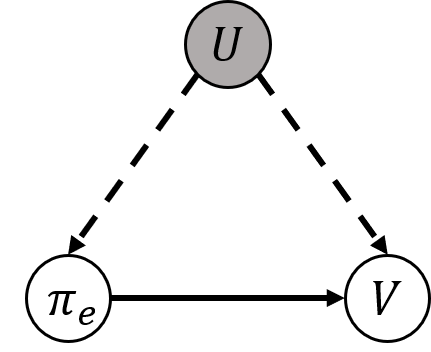}}\qquad \qquad
	\subfigure[\footnotesize Auxiliary expert]{\label{subfig:auxillary}\includegraphics[scale=0.35]{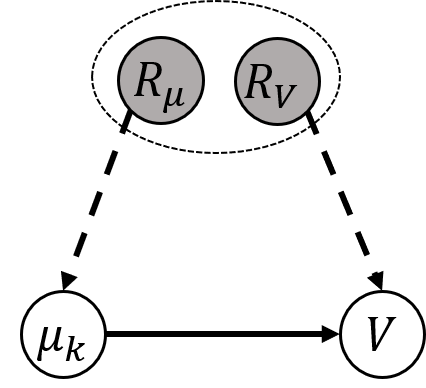}}\qquad \qquad
	\subfigure[\footnotesize Learner]{\label{subfig:learner}\includegraphics[scale=0.35]{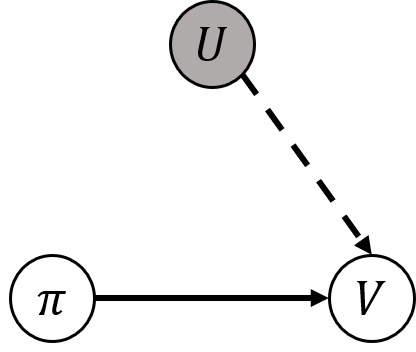}}
	\caption{\small Expert and learner causal graphs. (\textit{a}) represents the causal graph of expert; (\textit{b}) represents the auxiliary causal graph of expert where the context $U$ is replaced by $\{R_{\mu}, R_{V}\}$; (\textit{c}) denotes the causal graph of the learner. The gray nodes represent the unobservable context $U$, and the white nodes represent observable data. The arrows indicate the directions of the causal effects.}
	\label{fig:causal graph}
	\vspace{-6mm}
\end{figure}
\vspace{-4mm}
\subsection{Experience Transfer under Unobserved Contextual Information}
\label{sec:ProblemForm}
Let $\tau_D = \{\tau_1, \tau_2, \dots, \tau_N\}$ denote a dataset generated by the expert, where $\tau_i = \{s_0^i, a_0^i, s_1^i, a_1^i,$ $\dots,s_T^i, a_T^i, V^i\}$ represents the state-action trajectory collected during the $i$th execution of the context-aware policy $\pi_e^\ast$, and $V^i = \sum_{t=0}^T \gamma^t r^U(s_t^i, a_t^i)$ is the total reward received at the end of this episode. 
Specifically, to generate the expert data, at the beginning of episode $i$, we sample the context variable $U^i$ independently from the distribution $P(U)$. Then, given the sampled context $U^i$, the expert interacts with the environment using policy function $\pi^*_e(s_t^i, a_t^i, U^i)$ and collects data $\tau_i$. The contextual information $U^i$ is not recorded in dataset $\tau_D$. The expert data collection process is reflected in the causal graph in Figure~\ref{fig:causal graph}(a). 
Finally, the dataset $\tau_D$ without the contextual information $U$ is transferred to a learner agent with the goal to learn a context-unaware policy $\pi$, as described in the causal graph in FIgure~\ref{fig:causal graph}(c). Specifically, we aim to solve the following problem.
\begin{problem} \label{Prob:Transfer}
	Given the experience data $\tau_D$ generated by an expert agent that uses a context-aware policy $\pi_e^\ast$, design a transfer learning algorithm for a learner agent that uses the data $\tau_D$ excluding the contextual information to learn   the optimal context-unaware policy $\pi^\ast(s_t, a_t)$ that solves the optimization problem~\eqref{eqn:LearnerProblem} with only a few new data samples.
\end{problem}

In the absence of any unobserved contextual effects, Problem~\ref{Prob:Transfer} has been studied in \cite{Gail,BC}. In these works, a single policy function that generates the experience dataset $\tau_D$ with the maximum likelihood is computed and is directly transferred to the learner agent. When the dataset $\tau_D$ is generated as per the causal graph in Figure~\ref{fig:causal graph}(a), the methods in \cite{Gail,BC} essentially compute a mixture of the context-aware policy $\pi_e^\ast$ based on the context distribution $P(U)$. On the other hand, in the presence of unobserved contextual information, a possible approach to Problem~\ref{Prob:Transfer} is to first use the unsupervised clustering methods in \cite{multi_modal,infogail} on the dataset $\tau_D$ to obtain a cluster of basis policy functions $\{ \mu_k \}_{k = 1, 2, \dots, K}$, so that each function $\mu_k(s_t, a_t): \mathcal{S} \times \mathcal{A} \rightarrow [0,1]$ approximates the underlying true expert behavior $\pi_e^\ast$ associated with a given context. Then, we can select the basis policy function with the best empirical expected return in the dataset $\tau_D$. Next, we give an example to show that both these approaches can lead to a policy function that is far from the true optimal policy. 

\begin{table}[t]
	\begin{subtable}
		\centering
		\begin{tabular}{c | c | c}
			$r^U(a_t)$& $U = 0$ & $U = 1$ \\
			\hline
			$a_t = 0$ & 1 & 11 \\ 
			$a_t = 1$ & 3 & 10  \\ 										
		\end{tabular}
	\end{subtable}
	\hfill
	\begin{subtable}
		\centering
		\begin{tabular}{c | c | c | c | c }
			$P(a_t, r^U)$ & $r^U = 1$ & $r^U = 3$ &  $r^U = 10$ & $r^U = 11$ \\
			\hline
			$a_t = 0$ & 0.05 & 0 & 0 & 0.45 \\
			$a_t = 1$ & 0 & 0.45 & 0.05 & 0 \\
		\end{tabular}
	\end{subtable}
	\caption{\small Reward and Observational Distribution for the 2-arm Bandit Problem in Example~\ref{exm:Bandit}} \label{tab:Bandit}
	\vspace{-4mm}
\end{table}

\begin{example} \label{exm:Bandit}
	Consider a 2-arm bandit problem, where the action is binary, i.e., $a_t \in \{0,1\}$. At the beginning of each episode, a binary contextual variable $U \in \{0,1\}$ is sampled from a Bernoulli distribution $P(U = 0) = \frac{1}{2}$. The binary variable $U$ affects the reward $r^U(a_t)$ received during each episode according to Table~\ref{tab:Bandit} (a).
	The expert agent can observe the contextual variable $U$ and executes a context-aware policy $\pi_e^\ast(a_t, U)$, such that $\pi_e^\ast(a_t = 1, U = 0) = 0.9$ and $\pi_e^\ast(a_t = 0, U = 1) = 0.9$. As a result, the expert generates multiple data points $\{a_t, r_t\}$ to obtain an observational distribution $P(a_t, r_t)$ as shown in Table~\ref{tab:Bandit} (b). Denote the true expected reward when taking action $a_t = 0$ as $\mathbb{E}\big[r_t | do(a_t = 0)\big]$. Then, we have that $\mathbb{E}\big[r_t | do(a_t = 0)\big] = 6$ and $\mathbb{E}\big[r_t | do(a_t = 1)\big] = 6.5$. The true optimal policy is to select $a_t = 1$ with probability $1$.
	
	Given the underlying contextual variable distribution $P(U = 0) = 1/2$, the Direct Imitation strategy in \cite{Gail,BC} learns a policy $\pi_{DI}(a_t = 0) = 1/2$ and $\pi_{DI}(a_t = 1) = 1/2$, which is far from the true optimal policy. On the other hand,
	in this 2-arm bandit problem, 
	the basis policy function can be directly obtained as $\{\mu_1, \mu_2\}$, where $\mu_1(a_t = 0) = 1$ and $\mu_2(a_t = 1) = 1$. The empirical estimates of the expected rewards for these two basis policies can be computed using the observational distribution in Table~\ref{tab:Bandit} (b). Specifically, we have that $\mathbb{E}[r^U | a_t = 0] = 10$ and $\mathbb{E}[r^U | a_t = 1] = 3.7$. This suggests that we should transfer the basis policy function $\mu_1$ to the learner agent, which is the opposite of the true optimal policy. 	
\end{example}

Example~\ref{exm:Bandit} shows that selecting the mixture policy $\pi_{DI}$ or the basis policy function with the best observed rewards can both lead to sub-optimal policies. In addition, unlike the bandit problem in Example~\ref{exm:Bandit}, for general continuous RL problems, the optimal context-unaware policy $\pi^\ast$ may not belong to the set of basis policy functions obtained using the algorithms in \cite{multi_modal,infogail}. In the next section, we propose a new TL framework that can address these challenges.

\section{Continuous Transfer Learning with Unobserved Contextual Information}\label{sec:methods}
In this section, we propose a TL framework to solve Problem~\ref{Prob:Transfer}. First, we compute a set of basis policies from the dataset $\tau_D$. Then, we formulate a MAB problem that uses these basis policies as its arms and the learner agent can learn to select the best basis policy function to improve online. Finally, to reduce the exploration variance of the MAB problem, we constrain exploration using causal bounds on the accumulated rewards that can be achieved by these basis policy functions. 

\subsection{Multi-Armed Bandit Formulation}
The first step in defining a MAB model for Problem~\ref{Prob:Transfer} is to use the unsupervised clustering methods in \cite{multi_modal,infogail} to obtain a cluster of basis policy functions $\{ \mu_k \}_{k = 1, 2, \dots, K}$ from the expert dataset $\tau_D$, as discussed in Section~\ref{sec:ProblemForm}. Each function $\mu_k(s_t, a_t): \mathcal{S} \times \mathcal{A} \rightarrow [0,1]$ represents a distinct expert behavior under a specific context, e.g., driving cautiously or aggressively. We assume that the expert policy $\pi_e^\ast(s_t, a_t, U)$ is a mixture of the basis policies $\{\mu_k\}_{k = 1, 2, \dots, K}$, i.e., $\pi_e^\ast(s_t, a_t, U) = \sum_{k=1}^K \mu_k(s_t, a_t)P(\mu_k | U)$, where $P(\mu_k | U)$ denotes the conditional probability of selecting the basis policy $\mu_k$ when context $U$ is observed. Through this policy transformation, we obtain an auxiliary causal graph for the expert, as described in Figure~\ref{fig:causal graph}(b). And the dataset $\tau_D$ can be written as $\{m_i, V_i\}$, where $m_i \in \{1, 2, \dots, K\}$ represents the index of the basis poilcy function that generates the $i$th trajectory $\tau_i$ in dataset $\tau_D$ and $V_i$ denotes its accumulated rewards.
\begin{algorithm}[t]
	\begin{small}
		\caption{Transfer RL under Unobserved Contextual Information}
		\label{alg:TransferRL}
		\begin{algorithmic}[1]
			\REQUIRE The set of basis policy functions $\{\mu_k\}$, the causal bounds on the expected total rewards of each policy function $[l_k, h_k]$, a non-increasing function $f(i)$ and $T_k(0) = 1$ for all $k$.\\
			\STATE{	Remove any arm $k$ with $h_{k}<l_{max}$, where $l_{max} = \max_k \{l_k\}$.}
			\STATE{ Execute all basis policies $\mu_k$ once and record the resulting accumulated rewards $\{\hat{V}_k(0)\}$.}
			\FOR{ episode $i = 1,2, \dots$}
			\STATE{ For every basis policy function $\mu_k$, compute $\hat{H}_{k}(i)= \min \{H_{k}(i), h_{k}\}$, where
				\begin{align} \label{eqn:UCB}
					H_{k}(i)=\hat{V}_{k}(i-1)+\sqrt{\frac{2f(i)}{T_{k}(i-1)}};
			\end{align}}
			\STATE{ Select the basis policy function $\mu_{k^\ast}$ such that $k^\ast = \arg \max_{k}\hat{H}_{k}(i)$; \\}
			\STATE{ Execute the policy $\mu_{k^\ast}$ during episode $i$ and update it using policy optimization algorithms, e.g., PPO;}
			\STATE{ At the end of the episode, compute the accumulated rewards $V(i)$ and update $\hat{V}_{k^\ast}(i)$ as $\hat{V}_{k}(i) = \frac{1}{i+1} V(i) + \frac{i}{i+1} \hat{V}_{k}(i-1)$ if $k = k^\ast$, and $\hat{V}_{k}(i) =  \hat{V}_{k}(i-1)$ otherwise;}
			\STATE{ Update the number of times every basis policy function has been selected as $T_k(i) = T_{k}(i-1) + 1$ if $k = k^\ast$, and $T_k(i) = T_k(i-1)$ otherwise.}
			\ENDFOR
			\STATE{Output policy $\mu_{k^\ast}$.}
		\end{algorithmic}
	\end{small}
\end{algorithm}

Next, we use the set $\{\mu_k\}$ as initial policy parameters. Each one of these policy parameters gives rise to an arm in a MAB problem, which we solve using the Upper Confidence Bound (UCB) algorithm to obtain the optimal initial policy in $\{\mu_k\}$.
Then, we can improve this policy through online interactions using any existing policy optimization algorithm, e.g., the Proximal Policy Optimization (PPO) algorithm in \cite{PPO} with experience replay. This process is presented in Algorithm~\ref{alg:TransferRL}. Specifically, we start by evaluating the basis policy functions $\mu_k$ by executing them in the environment once and computing their accumulated rewards $\{\hat{V}_k(0)\}$ (line 2). Then, using these estimates of the total rewards we can compute the UCBs on the expected values of the basis policy functions $H_k(i)$ according to \eqref{eqn:UCB} (line 4). Next, we select the basis policy function $\mu_{k^\ast}$ that has the highest UCB on its expected value (line 5), execute this policy for the duration of one episode, and improve it using the PPO (line 6). At the end of each episode, we compute the total rewards of the improved basis policy $\mu_{k^\ast}$, and update the estimate of the expected value $\hat{V}_{k^\ast}$ using this policy (line 7) as well as the number of times $T_k$ this policy has been selected (line 8). Finally, Algorithm~\ref{alg:TransferRL} returns to line 4 and a new episode is initialized. 

Note that the UCB algorithm proposed in Algorithm~\ref{alg:TransferRL} gradually explores more frequently the most promising basis policy function that performs well initially and is improved at the fastest rate using new online data. 
However, this UCB algorithm may suffer large variance during exploration, which affects its learning rate. 
In the next section, we provide bounds on the expected returns of each basis policy function $\mathbb{E}\big[V | do(\mu_k)\big]$ using the expert data $\tau_D$ and transfer these bounds to the UCB algorithm to reduce its exploration variance.

\subsection{Causal Bound Constrained MAB}
When the contextual variable is missing from the experience data $\tau_D$, as discussed in \cite{causal_mab}, the true expected reward $\mathbb{E}[V | do(\mu)]$ cannot be identified. However, it is possible to obtain bounds on $\mathbb{E}[V | do(\mu)]$ using the observational data $\{m_i, V_i\}$, similar as \cite{zhang2020bounding}. To do so, define the unobserved contextual variable $U$ by a pair of random variables $(R_\mu, R_V)$, where $R_\mu$ has support $\{1, 2, \dots, K\}$ and $R_V$ has support $\mathbb{R}$. The random variable $R_\mu$ represents the basis policy function that the expert decides to adopt in practice, i.e., $\mu_{R_\mu} = f_\mu(R_\mu)$, while the random variable $R_V$ represents the accumulated rewards received when the expert executes policy $\mu_k$, i.e., $R_V = f_V(\mu_k, R_V)$. The causal relationship betweeen the observed expert policy $\mu$, the observed accumulated rewards $V$, and the unobserved random variable pair $(R_\mu, R_V)$ is shown in Figure~\ref{fig:causal graph}(b). If the RL problem has bounded rewards and is episodic, then the random accumulated rewards $R_V$ have a compact support $R_V \in [V_l, V_u]$. Then, given the observed distribution $P(\mu, V)$, we can obtain the following bounds on the causal effect $\mathbb{E}[V | do(\mu)]$.


\begin{thm} \label{thm:CB}
	If $R_V \in [V_l, V_u]$, then we have that 
	\begin{align} \label{eqn:CB_1}
		\mathbb{E}\big[ V | \mu_k \big] P(\mu_k) + (1- P(\mu_k)) V_l\leq \mathbb{E}\big[ V | do(\mu_k) \big] \leq  \mathbb{E}\big[ V | \mu_k \big] P(\mu_k) + (1 - P(\mu_k)) V_u.
	\end{align}
	Furthermore, let $\mathbb{E}\big[ V | \mu_j, R_\mu = k \big]$ represent the expected return when the expert agent counterfactually executes policy $\mu_j$ instead of the policy $\mu_k$ that was used by the expert to generate the data, so that $R_\mu = k$. 
	If $\; \mathbb{E}\big[ V | \mu_j, R_\mu = k \big] \leq \mathbb{E}\big[ V | \mu_k, R_\mu = k \big]$ for all $j \neq k$, then we have that
	\vspace{-2mm}
	\begin{align} \label{eqn:CB_2}
		\mathbb{E}\big[ V | \mu_k \big] P(\mu_k) + (1- P(\mu_k)) V_l\leq \mathbb{E}\big[ V | do(\mu_k) \big] \leq  \sum_{j=1}^K \mathbb{E}\big[ V | \mu_j \big] P(\mu_j).
	\end{align}	
\end{thm}

\begin{proof}
	Recall that 
	\begin{align} \label{eqn:CausalEffect}
		\mathbb{E}[V | do(\mu_k)] =  \int_{R_\mu = k, R_V} f_V(\mu_k, R_V) dP(R_\mu = k, R_V) + \int_{R_\mu \neq k, R_V} f_V(\mu_k, R_V) dP(R_\mu \neq k, R_V),
	\end{align}
	where the first term in \eqref{eqn:CausalEffect} represents the expected accumulated rewards of the trajectories observed in the  dataset $\tau_D$ that are generated using policy $\mu_k$, i.e., 
	\begin{align} \label{eqn:ObservedReward}
		\int_{R_\mu = k, R_V} f_V(\mu_k, R_V) dP(R_\mu = k, R_V) & = P(R_\mu = k) \int_{R_V} f_V(\mu_k, R_V) dP(R_V | R_\mu = k) \nonumber \\
		& = \mathbb{E}\big[ V | \mu_k \big] P(\mu_k).
	\end{align}
	On the other hand, the second term in \eqref{eqn:CausalEffect} represents the counterfactual accumulated rewards received when executing policy $\mu_k$, although other policy functions $\mu_{j \neq k}$ are actually executed in the dataset $\tau_D$. Since policy $\mu_k$ was not actually executed when $R_\mu \neq k$, its counterfactual accumulated rewards are never observed in expert dataset $\tau_D$,
	and we can only bound $\int_{R_\mu \neq k, R_V} f_V(\mu_k, R_V) dP(R_\mu \neq k, R_V)$ using the bounds on the random variable $R_V$, i.e.,
	\begin{align} \label{eqn:BoundCounterfact}
		V_l \big(1 - P(\mu_k)\big) \leq \int_{R_\mu \neq k, R_V} f_V(\mu_k, R_V) dP(R_\mu \neq k, R_V) \leq  V_u \big(1 - P(\mu_k)\big).
	\end{align}
	Substituting the bounds \eqref{eqn:ObservedReward} and \eqref{eqn:BoundCounterfact}  into \eqref{eqn:CausalEffect}, we obtain the bound in \eqref{eqn:CB_1}.
	
	Now, assuming that $\mathbb{E}\big[ V | \mu_j, R_\mu = k \big] \leq \mathbb{E}\big[ V | \mu_k, R_\mu = k \big]$ for all $j \neq k$, we have that $\int_{R_\mu = k, R_V}$ $f_V(\mu_j, R_V) dP(R_\mu = k, R_V) \leq \int_{R_\mu = k, R_V} f_V(\mu_k, R_V) dP(R_\mu = k, R_V)$
	for every $j \neq k$. Using this bound, we obtain that
	\begin{align} \label{eqn:proof_2}
		\int_{R_\mu \neq k, R_V} f_V(\mu_k, R_V) dP(R_\mu \neq k, R_V) \leq \sum_{j \neq k} f_V(\mu_j, R_V) dP(R_\mu = j, R_V).
	\end{align}
	Combining the bound~\eqref{eqn:proof_2} with \eqref{eqn:CausalEffect}, we obtain the bound in \eqref{eqn:CB_2}. The proof is complete.
\end{proof}

Note that the assumption $\mathbb{E}\big[ V | \mu_j, R_\mu = k \big] \leq \mathbb{E}\big[ V | \mu_k, R_\mu = k \big]$ in Theorem~\ref{thm:CB} suggests that the basis policy $\mu_k$ implemented by the expert is better than any other policy in expectation. This holds when the expert policy is close to the optimal context-aware poilcy $\pi_e^\ast(a_t, s_t, U)$. Same as in \cite{causal_mab}, here too we can incorporate in Algorithm~\ref{alg:TransferRL} the bounds $[l_k, h_k]$ on the causal effect $\mathbb{E}\big[ V | do(\mu_k) \big]$ computed using Theorem~\ref{thm:CB} for every basis policy function $\mu_k$. 
To do so, we compare the bound $H_k$ in \eqref{eqn:UCB} computed using online data to the upper bound $h_k$, and select the tighter one  as an estimate for the UCB $\hat{H}_k$ (line 4 in Algorithm~\ref{alg:TransferRL}). Observe that during the early learning stages, the UCB $H_k$ computed using new online data can be over-optimistic for suboptimal basis policies. The causal bound $h_k$ obtained using the expert's experience can tighten these bounds and, therefore, reduce the number of queries to suboptimal basis policies. This can significantly improve the efficiency of the data collection for the learner by reducing the variance of the UCB exploration, as we also show in Section~\ref{sec:experiment}.

\section{Numerical Experiments}\label{sec:experiment}
In this section, we demonstrate the effectiveness of the proposed transfer RL algorithm on an autonomous racing car example. Specifically, the goal is to transfer driving data from an expert driver that has a global mental view of the race track (e.g., from past experience driving on this track) to a learner autopilot that can only perceive the race track locally using the vehicle's on-board sensors. In practice, the expert's decisions will be affected by their mental view of the race track, e.g., the anticipated shape of the track ahead. However, this information is not recorded in the expert's data and, therefore, is not available to the learner autopilot. As such, it constitutes unobservable context.

\begin{figure}[t]
	\centering
	\subfigure[\label{fig:env}]{
		\includegraphics[width=0.38\linewidth]{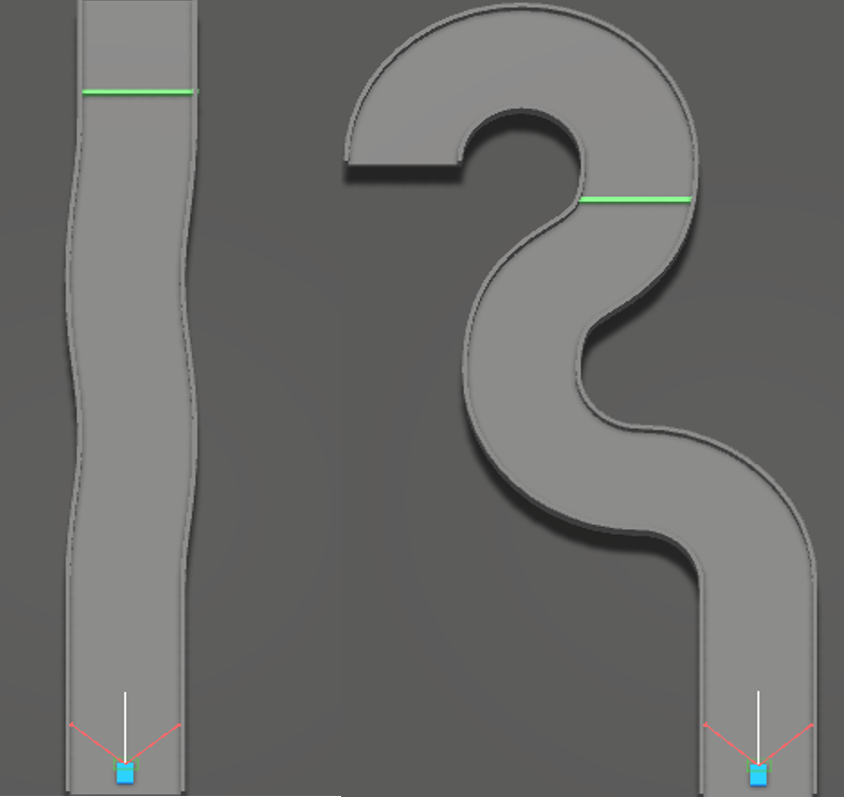}		
	}
	\subfigure[\label{fig:comparison}]{
		\includegraphics[width = 0.5\linewidth]{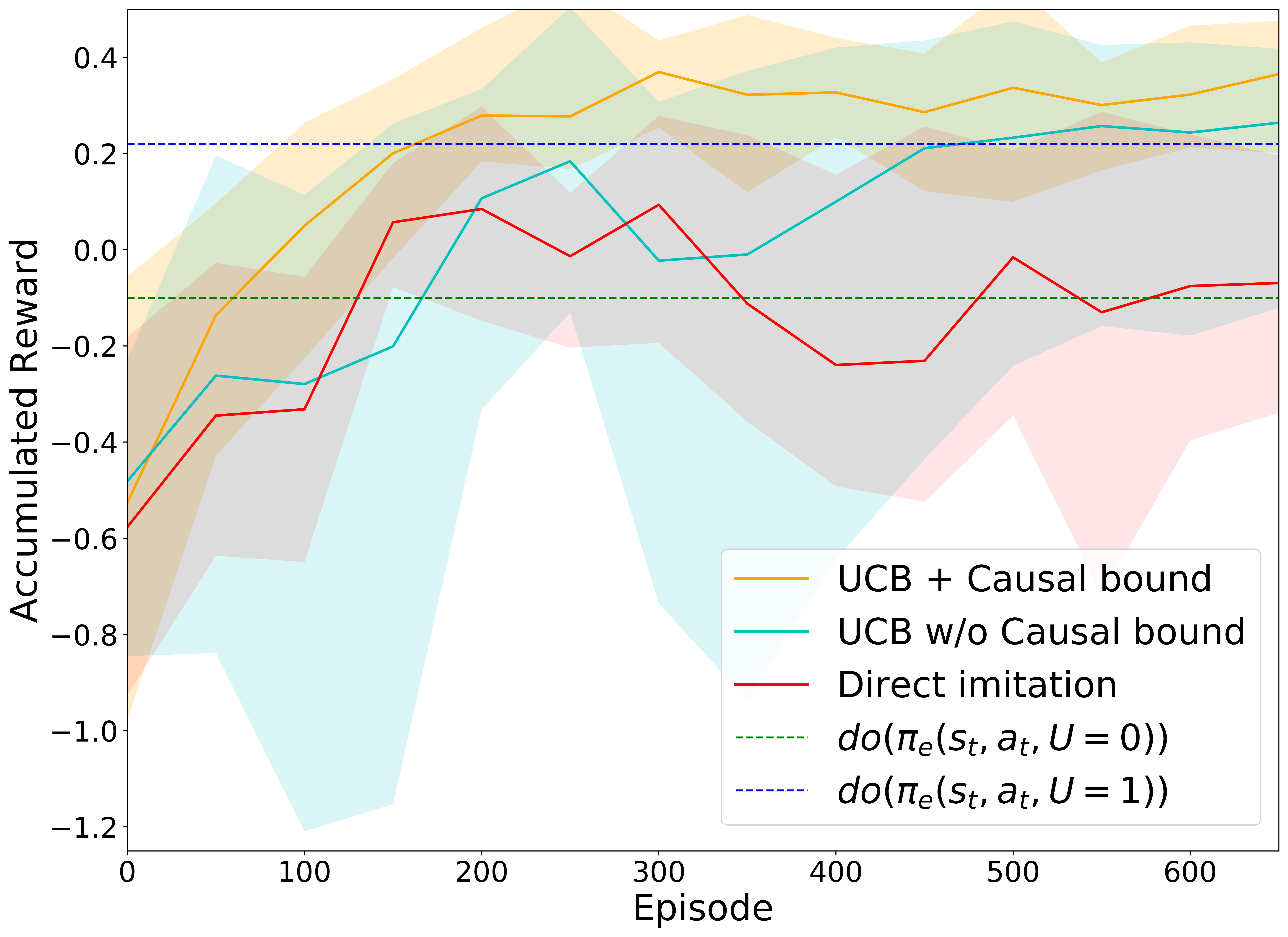}
	} 
	\caption{\small In (a), two types of tracks are presented. The straight track denotes the environment under $U = 0$ and the curved track denotes the environment under $U = 1$; In (b), the performance of Algorithm~\ref{alg:TransferRL} (orange), Algorithm~\ref{alg:TransferRL} without causal bounds (blue) and transferring a single policy learnt by direct imitation learning (red) are presented. Each curve is obtained by running 10 trials.}
	\label{fig:env_comparison}
	\vspace{-6mm}
\end{figure}

We simulate the above problem using the Unity Physics Engine, \cite{juliani2018unity}. Specifically, we consider the two types of tracks shown in Figure~\ref{fig:env}. 
Let $U = 0$ denote the straight track and $U = 1$ denote the curved track. We assume that the contextual variable $U$ is sampled from a Bernoulli distribution $P(U = 0) = 0.7$ at the beginning of each episode. The green lines in both tracks represent the finish line. In Figure~\ref{fig:env}, the vehicle (blue box) is equipped with an onboard sensor (three lines emitting from the car) that detects the track boundaries, if within sensor range, and their distance to the vehicle. The vehicle takes actions in the set $\mathcal{A} := \{0,1,2,3,4\}$, where action $0$ means "maintain current speed", action $1$ means "accelerate", action $2$ means "decelerate", action $3$ means "turn clockwise", and action $4$ means "turn counterclockwise". The policy of the learner autopilot is denoted by $\pi(s_t, a_t)$, where the state $s_t$ represents the vehicle's orientation, velocity, last-step action, and the sensor measurements. On the other hand, the policy of the expert driver is denoted by $\pi_e^\ast(s_t, a_t, U)$, where the context variable $U$ denotes the type of the track that the expert has experience driving on.
The goal is to drive the vehicle to the finish line in the shortest period of time while avoiding collisions with the track boundaries. 
To achieve this, let the learner autopilot receive a reward $r = 1$ when it reaches the finish line and $r = -0.08$ when the vehicle collides with the track boundaries. Otherwise, the agent receives a reward $r = -0.003$. Each episode ends when the vehicle reaches the finish line or when 500 time steps have elapsed.
We trained the expert policies $\pi_e^\ast(s_t, a_t, U = 0 \text{ or } 1)$ for each track independently using the PPO algorithm. Then, we executed the expert policies 300 times to construct the dataset $\tau_D$ as discussed in Section~\ref{sec:ProblemForm}.

After constructing the dataset $\tau_D$, we transfer it to the learner autopilot and cluster the trajectories in $\tau_D$ into two sets using the K-means algorithm, where the number of clusters can be determined using Principle Component Analsysis. Then, using behavior cloning \cite{BC}, the learner agent learns two basis policy functions $\{\mu_0,\mu_1\}$ from the two trajectory sets, where the policy $\mu_0$ (or $\mu_1$) approximates the expert policy $\pi_e^\ast(s_t, a_t, U=0 \text{ or } 1)$. After clustering, the expert dataset $\tau_D$ can be written as $\{m_i, V_i\}$, where $m_i \in \{0,1\}$ denotes the index of the basis policy function that generates the $i$th trajectory $\tau_i$ in $\tau_D$. Using the data $\{m_i, V_i\}$, the learner agent can compute the emprical mean of the returns $\mathbb{E}\big[V|\mu_k\big]$ and the causal bounds on the true expected returns, as in \eqref{eqn:CB_2}.
In addition, we can evaluate the true expected returns $\mathbb{E}\big[V|do(\mu_0 \text{ or } \mu_1)\big]$ of the expert policies $\pi_e^\ast(s_t, a_t, U = 0 \text{ or } 1)$ through simulations. These results are presented in Table~\ref{tab:demo_data}. 
Observe that the basis policy function $\mu_1$ achieves higher true expected accumulated rewards compared to the policy $\mu_0$, i.e., $\mathbb{E}\big[V|do(\mu_1)\big] > \mathbb{E}\big[V|do(\mu_0)\big]$. Therefore, $\mu_1$ is a better policy to select.
However, if the basis policy function is selected based on its empirical expected returns $\mathbb{E}\big[V|\mu_k\big]$, the wrong policy $\mu_0$ will be selected. Furthermore, the causal bounds~\eqref{eqn:CB_2} in Table~\ref{tab:demo_data} correctly characterize the range of 
the value of $\mathbb{E}\big[V|do(\mu_k)\big]$.
\begin{table}[t]
	\centering
	\begin{tabular}{c | c | c | c | c }
		\hline 
		& $\mathbb{E}[V | \mu_{k} ]$ &$\mathbb{E}[V|do(\mu_{k})]$ & $l_{k}$ & $h_{k}$  \\
		\hline 
		$do(\mu_{0})$ & 0.479 & -0.10 &-0.89 & 0.407\\
		
		$do(\mu_{1})$ & 0.244  & 0.22 &-2.69 & 0.407\\		
		\hline
	\end{tabular}
	\caption{\small Empirical mean, true mean and causal bounds on the accumulated rewards.}
	\label{tab:demo_data}
	\vspace{-2mm}
\end{table}

Having constructed the basis policy functions $\{\mu_k\}$, we then compared Algorithm~\ref{alg:TransferRL} to a MAB method on the set of basis functions without causal bounds, and the direct imitation learning method in \cite{BC} that learns a single policy from the expert dataset $\tau_D$ and transfers this policy directly to the learner agent. The results are presented in Figure~\ref{fig:comparison}. Observe that TL in the case of MAB without causal bounds (blue curve) or direct imitation learning (red curve) suffers large variance during online training. This is because both the policy $\pi_{DI}$ obtained from $\tau_D$ using direct imitation learning and the basis policy function $\mu_0$ selected by MAB are far from the optimal policy. Improving these suboptimal polices online generally results in higher variance during training compared to improving a good quality policy, e.g., the basis policy function $\mu_1$. 
On the other hand, our proposed Algorithm~\ref{alg:TransferRL} uses the causal bounds to reduce the variance of MAB exploration and finds the best basis policy function much faster than MAB without causal bounds. To elaborate, when the UCB on the expected return of the suboptimal basis policy function computed using \eqref{eqn:UCB} is over-optimistic, the causal bounds project this UCB to a tigher bound. This reduces the number of times the suboptimal basis policy function is explored. Consequently, Algorithm~\ref{alg:TransferRL} (orange curve) can find the optimal policy within much fewer episodes compared to the other two methods and enjoys the lowest training variance, as shown in Figure~\ref{fig:comparison}.

\section{Conclusion}\label{sec:conclusion}
In this paper, we considered a transfer RL problem in continuous state and action spaces under unobserved contextual information. The goal was to use data generated by an expert agent that knows the context to learn an optimal policy for a learner agent that can not observe the same context, using only a few new data samples. To this date, such problems are typically solved using imitation learning that assumes that both the expert and learner agents have access to the same information. However, if the learner does not know the expert context, using expert data alone will result in a biased learner policy and will require many new data samples to improve. To address this challenge, in this paper, we proposed a MAB method that uses the expert data to train initial basis policies that serve as the arms in the MAB and to compute causal bounds on the accumulated rewards of these basis policies to reduce the MAB exploration variance and improve the learning rate. We provided numerical experiments on an autonomous driving example that showed that our proposed transfer RL method improved the learner's policy faster compared to imitation learning methods and enjoyed much lower variance during training. 

\acks{This work is supported in part by AFOSR under award \#FA9550-19-1-0169 and by NSF under award CNS-1932011.}

\bibliography{citation}

\end{document}